\documentclass[11pt,a4paper]{article}

\usepackage{hyperref}
\usepackage{AVTcommands}
\usepackage{AVTenvironments}
\usepackage{microtype}

\usepackage{emnlp2020}
\aclfinalcopy

\usepackage{dblfloatfix}

\usepackage{amsthm}
\newtheoremstyle{theorem}{2\topsep}{\topsep}{\slshape}{}{\bfseries}{.}{5pt plus 1pt minus 1pt}{}
\theoremstyle{theorem}
\newtheorem{theorem}{Theorem}

\newtheorem{definition}[theorem]{Definition}
\newtheorem{result}[theorem]{Result}
\newtheorem{proposition}[theorem]{Proposition}
\newtheorem{remark}[theorem]{Remark}
\newtheorem*{assumption*}{Assumption}
\newtheorem{Algorithm}{Algorithm}

\renewcommand{\vec}{\ifmmode\expandafter\boldsymbol\fi} 
\newcommand{\mPhi}{\ifmmode\m\Phi\else Phi\fi} 
\newcommand{\mPsi}{\ifmmode\m\Psi\else Psi\fi} 

\usepackage{array}
\newcolumntype{C}[1]{>{\centering\arraybackslash}p{#1}}

\def\UrlFont{\scshape}

\let\textcite\cite
\let\cite\citep 

\usepackage{times}

\begin{document}

\title{Sparse Parallel Training of Hierarchical Dirichlet Process Topic Models}
\author{Alexander Terenin \\
  Imperial College London \And
  M{\aa}ns Magnusson \\
  Uppsala University\\and Aalto University \And
  Leif Jonsson \\
  Ericsson AB \\ and Link\"{o}ping University}
\date{}

\maketitle

\begin{abstract}
To scale non-parametric extensions of probabilistic topic models such as Latent Dirichlet allocation to larger data sets, practitioners rely increasingly on parallel and distributed systems.
In this work, we study data-parallel training for the hierarchical Dirichlet process (HDP) topic model.
Based upon a representation of certain conditional distributions within an HDP, we propose a doubly sparse data-parallel sampler for the HDP topic model.
This sampler utilizes all available sources of sparsity found in natural language---an important way to make computation efficient.
We benchmark our method on a well-known corpus (PubMed) with 8m documents and 768m tokens, using a single multi-core machine in under four days.
\end{abstract}

\section{Introduction}

Topic models are a widely-used class of methods that allow practitioners to identify latent semantic themes in large bodies of text in an unsupervised manner.
They are particularly attractive in areas such as history \cite{yang11,wang12}, sociology \cite{dimaggio13}, and political science \cite{roberts14}, where a desire for careful control of structure and prior information incorporated into the model motivates one to adopt a Bayesian approach to learning.
In these areas, large corpora such as newspaper archives are becoming increasingly available \cite{ehrmann20}, and models such as latent Dirichlet allocation (LDA) \cite{blei03} and its nonparametric extensions \cite{teh06a, teh06b, hu12, paisley15} are widely used by practitioners.
Moreover, these models are emerging as a component of data-efficient language models \cite{guo19}.
Training topic models efficiently entails two requirements.
\1 Expose sufficient parallelism that can be taken advantage of by the hardware.
\2 Utilize sparsity found in natural language to control memory requirements and computational complexity.
\0

\begin{table*}[t!]
\begin{center}
\captionfont
\begin{tabular}{c l c c l}
\hline
Symbol & Description && Symbol & Description
\\
\cline{1-2} \cline{4-5}
$V$ & Vocabulary size
&&
$\v\Psi: 1 \x \infty$ & Global distribution over topics
\\
$D$ & Total number of documents
&&
$\m{\Theta} : D \x \infty$ & Document-topic probabilities
\\
$N$ & Total number of tokens
&&
$\v{\theta}_d : 1 \x \infty$ & Topic probabilities for document $d$
\\
$v(i)$ & Word type for token $i$
&&
$\m{m} : D \x \infty$ & Document-topic sufficient statistic
\\
$d(i)$ & Document for token $i$
&&
$\m{\Phi} : \infty \x V$ & Topic-word probabilities
\\
$w_{i,d}$ & Token $i$ in document $d$
&&
$\v{\phi}_k : 1 \x V$ & Word probabilities for topic $k$
\\
$b_{i,d}$ & Global topic draw indicator for $w_{i,d}$
&&
$\m{n} : \infty \x V$ & Topic-word sufficient statistic
\\
$z_{i,d}$ & Topic indicator for token $i$ in $d$
&&
$\v{l}:1\x\infty$ & Global topic latent sufficient \rlap{statistic}\hspace*{6ex}
\\
$K^*$ & Index for implicitly-represented \rlap{topics}\hspace*{4ex}
&&
$\alpha,\v\beta,\gamma$ & Prior concentration for $\v\theta_d$, $\v\phi_k$, $\v\Psi$
\\
\hline
\end{tabular}
\end{center}
\caption{Notation for the HDP topic model. Sufficient statistics are conditional on the algorithm's current iteration. Bold symbols refer to matrices, bold italics refer to vectors, possibly countably infinite.}
\end{table*}

In this work, we focus on the \emph{hierarchical Dirichlet process} (HDP) topic model of \textcite{teh06a}, which we review in Section \ref{sec:alg}.
This model is a simple non-trivial extension of LDA to the nonparametric setting. 
This parallel implementation provides a blueprint for designing massively parallel training algorithms in more complicated settings, such as nonparametric dynamic topic models \cite{ahmed10} and tree-based extensions \cite{hu12}.

Parallel approaches to training HDPs have been previously introduced by a number of authors, including \textcite{newman09}, \textcite{wang11}, \textcite{williamson13}, \textcite{chang14} and \textcite{ge15}.
These techniques suit various settings: some are designed to explicitly incorporate sparsity present in natural language and other discrete spaces, while others are intended for HDP-based continuous mixture models.
\textcite{gal14} have pointed out that some methods can suffer from load-balancing issues, which limit their parallelism and scalability.
The largest benchmark of parallel HDP training performed to our awareness is by \textcite{chang14} on the 100m-token \textsc{NYTimes} corpora.
Throughout this work, we focus on Markov chain Monte Carlo (MCMC) methods---empirically, their scalability is comparable to variational methods \cite{magnusson18,hoffman19}, and, subject to convergence, they yield the correct posterior.

Our contributions are as follows.
We propose an augmented representation of the HDP for which the topic indicators can be sampled in parallel over documents.
We prove that, under this representation, the global topic distribution $\v\Psi$ is conditionally conjugate given an auxiliary parameter $\v{l}$.
We develop fast sampling schemes for $\v\Psi$ and $\v{l}$, and propose a training algorithm with a per-iteration complexity that depends on the minima of two sparsity terms---it takes advantage of both document-topic and topic-word sparsity simultaneously.

\section{Partially collapsed Gibbs sampling for hierarchical Dirichlet processes}
\label{sec:alg}

The hierarchical Dirichlet process topic model \cite{teh06a} begins with a global distribution $\v\Psi$ over topics.
Documents are assumed exchangeable---for each document $d$, the associated topic distribution $\v\theta_d$ follows a Dirichlet process centered at $\v\Psi$.
Each topic is associated with a distribution of tokens $\v\phi_k$.
Within each document, tokens are assumed exchangeable (bag of words) and assigned to topic indicators $z_{i,d}$.
For given data, we observe the tokens $w_{i,d}$.

We thus arrive at the GEM representation of a HDP, given by equation (19) of \textcite{teh06a} as
\< \label{eqn:gem-representation-1}
\v\Psi &\~[GEM](\gamma)
\\ \label{eqn:gem-representation-2}
\v\theta_d \given \v\Psi &\~[DP](\alpha, \v\Psi)
\\ \label{eqn:gem-representation-3}
\v\phi_k &\~[Dir](\v\beta)
\\ \label{eqn:gem-representation-4}
z_{i,d} \given \v\theta_d &\~[Discrete](\v\theta_d)
\\ \label{eqn:gem-representation-5}
w_{i,d} \given z_{i,d}, \m\Phi &\~[Discrete](\v \phi_{z_{i,d}})
\>
where $\alpha, \v\beta, \gamma$ are prior hyperparameters.

\subsection{Intuition and augmented representation}

At a high level, our strategy for constructing a scalable sampler is as follows.
Conditional on $\v\Psi$, the likelihood in equations \eqref{eqn:gem-representation-1}--\eqref{eqn:gem-representation-5} is the same as that of LDA.
Using this observation, the Gibbs step for $\v{z}$, which is the largest component of the model, can be handled efficiently by leveraging insights on sparse parallel sampling from the well-studied LDA literature \cite{yao09, li14, magnusson18, terenin19b}.
For this strategy to succeed, we need to ensure that all Gibbs steps involved in the HDP under this representation are analytically tractable and can be computed efficiently.
For this, the representation needs to be modified.

To begin, we integrate each $\v\theta_d$ out of the model, which by conjugacy \cite{blackwell73} yields a P\'{o}lya sequence for each $\v{z}_d$.
By definition, given in Appendix \ref{apdx:sufficiency}, this sequence is a mixture distribution with respect to a set of Bernoulli random variables $\v{b}_d$, each representing whether $z_{i,d}$ was drawn from $\v\Psi$ or from a repeated draw in the P\'{o}lya urn.
Thus, the HDP can be written
\<
\v\Psi &\~[GEM](\gamma)
\\
b_{i,d} &\~[Ber]\del[1]{{\textstyle\frac{\alpha}{i-1+\alpha}}}
\\
\v\phi_k &\~[Dir](\v\beta)
\\
\v{z}_{d} \given \v{b}_d, \v\Psi &\~[PS](\v\Psi, \v{b}_d)
\\
w_{i,d} \given z_{i,d} &\~[Discrete](\v \phi_{z_{i,d}})
\>
where $\f{PS}(\v\Psi,\v{b}_d)$ is a P\'{o}lya sequence, defined in Appendix \ref{apdx:sufficiency}.
This representation defines a posterior distribution over $\v{z},\m\Phi,\v\Psi,\v{b}$ for the HDP.
To derive a Gibbs sampler, we calculate its full conditionals.

\subsection{Full conditionals for $\vec{z}$, $\mPhi$, and $\vec{b}$}

The full conditionals $\v{z}\given\m\Phi,\v\Psi$ and $\m\Phi\given\v{z},\v\Psi$, with $\v{b}$ marginalized out, are essentially those in partially collapsed LDA \cite{magnusson18, terenin19b}. 
They are
\<
\P(z_{i,d} = k &\given \v{z}_{-i,d},\m\Phi,\v\Psi)
\\
&\propto \phi_{k,v(i)}\sbr{\alpha \Psi_k + m^{-i}_{d,k}}
\>
where $v(i)$ is the word type for word token $i$, and
\[
\v\phi_k \given \v{z} \~[Dir](\v\beta + \v{n}_k)
\]
where $m^{-i}_{d,k}$ denotes the document-topic sufficient statistic with index $i$ removed, and $\v{n}_k$ is the topic-word sufficient statistic.
Note the number of possible topics and full conditionals $\v\phi_k\given\v{z}$ here is countably infinite.
The full conditional for each~$b_{i,d}$~is
\<
\label{eqn:bernoullis}
\P(b_{i,d} = 1 &\given \v{z}_d, \v\Psi, \v{b}_{-i,d}) 
\\
&= \frac{\alpha \Psi_{z_{i,d}}}{\alpha \Psi_{z_{i,d}} + \sum_{j=1}^i \1_{z_{j,d}}(z_{i,d})}
.
\>
The derivation, based on a direct application of Bayes' Rule with respect to the probability mass function of the P\'{o}lya sequence, is in Appendix \ref{apdx:sufficiency}.

\subsection{The full conditional for $\mPsi$}
\label{sec:psi}

To derive the full conditional for $\v\Psi$, we examine the prior and likelihood components of the model.
It is shown in Appendix \ref{apdx:sufficiency} that the likelihood term $\v{z}_d\given\v{b}_d,\v\Psi$ may be written
\<
p&(\v{z}_d \given \v{b}_d, \v\Psi) 
\\
&= \underbracket[0.1ex]{\prod_{\substack{i = 1 \\ b_{i,d} \neq 1}}^{N_d} \sum_{j=1}^{i-1} \frac{1}{i-1} \1_{z_{j,d}}(z_{i,d})}_{\smash{\t{doesn't enter posterior}}}\prod_{\substack{i = 1 \\ b_{i,d} = 1}}^{N_d} \prod_{k=1}^\infty \Psi_k^{\1_k(z_{i,d})}
.
\nonumber
\>
The first term is a multiplicative constant independent of $\v\Psi$ and vanishes via normalization.
Thus, the full conditional $\v\Psi\given\v{z},\v{b}$ depends on $\v{z}$ and $\v{b}$ only through the sufficient statistic $\v{l}$ defined by
\<
l_k = \sum_{d=1}^D \sum_{\substack{i=1\\\smash{b_{i,d} = 1}}}^{N_d} \1_{z_{i,d} = k}
\>
and so we may suppose without loss of generality that the likelihood term is categorical.
Under these conditions, we prove the full conditional for $\v\Psi$ admits a stick-breaking representation.

\begin{proposition}
\label{prop:gem}
Without loss of generality, suppose
\<
\v\Psi &\~[GEM](\gamma)
&
\v{x}\given\v\Psi &\~[Discrete](\v\Psi)
.
\>
Then $\v\Psi\given\v{x}$ is given by
\?
\label{eqn:post-gem-1}
\!\!\!\!\Psi_k=\varsigma_k \prod_{i=1}^{k-1} (1 - \varsigma_i)
\quad
\varsigma_k \~[Beta](a^{(\v\Psi)}_k\!,b^{(\v\Psi)}_k)
\!\!\!
\\
\label{eqn:post-gem-2}
a^{(\v\Psi)}_k = 1 + l_k
\,\,\quad\,\,
b^{(\v\Psi)}_k = \gamma + \sum_{i=k+1}^{\smash{\infty}} l_i
\?
where $\v{l}$ are the empirical counts of $\v{x}$.
\end{proposition}
\begin{proof}
Appendix \ref{apdx:gem}.
\end{proof}

This expression is similar to the stick-breaking representation of a Dirichlet process $\f{DP}(\cdot,F)$---however, it has different weights and does not include random atoms drawn from $F$ as part of its definition---see Appendix \ref{apdx:gem} for more details.
Putting these ideas together, we define an infinite-dimensional parallel Gibbs sampler.

\begin{Algorithm} \label{alg:pc-hdp}
Repeat until convergence.
\1* Sample $\v{\phi}_k \~[Dir](\v{n}_k + \v{\beta})$ in parallel over topics for $k=1,..,\infty$.
\2* Sample $z_{i,d} \propto \phi_{k,v(i)}\,\alpha\,\Psi_k + \phi_{k,v(i)}\, m_{d,k}^{-i}$ in parallel over documents for $d=1,..,D$.
\3* Sample $b_{i,d}$ according to equation \eqref{eqn:bernoullis} in parallel over documents for $d=1,..,D$.
\4* Sample $\v\Psi$ according to equations \eqref{eqn:post-gem-1}--\eqref{eqn:post-gem-2}.
\0*
\end{Algorithm}

\noindent
Algorithm \ref{alg:pc-hdp} is completely parallel, but cannot be implemented as stated due to the infinite number of full conditionals for $\m\Phi$, as well as the infinite product used in sampling $\v\Psi$.
We now bypass these issues by introducing an approximate finite-dimensional sampling scheme.

\subsection{Finite-dimensional sampling of $\mPsi$ and $\mPhi$}

By way of assuming $\v\Psi\~[GEM](\gamma)$, an HDP assumes an infinite number of topics are present a priori, with the number of tokens per topic decreasing rapidly with the topic's index in a manner controlled by $\gamma$.
Thus, under the model, a topic with a sufficiently large index should contain no tokens with high probability.

We thus propose to approximate $\v\Psi$ by projecting its tail onto a single flag topic $K^*$, which stands for all topics not explicitly represented as part of the computation.
This can be done by by deterministically setting $\varsigma_{K^*} = 1$ in equation \eqref{eqn:post-gem-1}.
The resulting finite-dimensional $\v\Psi$ will be the correct posterior full conditional for the finite-dimensional generalized Dirichlet prior considered previously in Section \ref{sec:psi}.
Hence, this finite-dimensional truncation forms a Bayesian model in its own right, which suggests it should perform reasonably well.
From an asymptotic perspective, \textcite{ishwaran01} have shown that the approximation is almost surely convergent and, therefore, well-posed.

Once this is done, $\v\Psi$ becomes a finite vector of length $K^*$, and only $K^*$ rows of $\m\Phi$ need to be explicitly instantiated as part of the computation.
This instantiation allows the algorithm to be defined on a fixed finite state space, simplifying bookkeeping and implementation.

From a computational efficiency perspective, the resulting value $K^*$ takes the place of $K$ in partially collapsed LDA. However, it \emph{cannot} be interpreted as the number of topics in the sense of LDA.
Indeed, LDA implicitly assumes that $\v\Psi = \f{Unif}(1,..,K)$ deterministically---i.e., that every topic is assumed a priori to contain the same number of tokens. In contrast, the HDP model learns this distribution from the data by letting $\v\Psi \~[GEM](\gamma)$.

If we allow the state space to be resized when topic $K^*$ is sampled, then following \textcite{papaspiliopoulos08}, it is possible to develop truncation schemes which introduce no error.
Since this results in more complicated bookkeeping which reduces performance, we instead fix $K^*$ and defer such considerations to future work.
We recommend setting $K^*$ to be sufficiently large that it does not significantly affect the model's behavior, which can be checked by tracking the number of tokens assigned to the topic $K^*$.

\subsection{Sparse sampling of $\mPhi$ and $\vec{z}$}

To be efficient, a topic model needs to utilize the sparsity found in natural language as much as possible.
In our case, the two main sources of sparsity are as follows.
\1 \emph{Document-topic sparsity}: most documents will only contain a handful of topics.
\2 \emph{Topic-word sparsity}: most word types will not be present in most topics.
\0
We thus expect the document-topic sufficient statistic $\m{m}$ and topic-word sufficient statistic $\m{n}$ to contain many zeros.
We seek to use this to reduce sampling complexity.
Our starting point is the Poisson P\'{o}lya Urn sampler of \textcite{terenin19b}, which presents a Gibbs sampler for LDA with computational complexity that depends on the minima of two sparsity coefficients representing document-topic and topic-word sparsity---such algorithms are termed \emph{doubly sparse}.
The key idea is to approximate the Dirichlet full conditional for $\v\phi_k$ with a Poisson P\'{o}lya Urn (PPU) distribution defined by
\<
\!\!\!\phi_{k,v} &\!=\! \frac{\varphi_{k,v}}{\sum_{v=1}^V \varphi_{k,v}}
&
\varphi_{k,v} &\!\~\!\f{Pois}(\beta_{k,v}\!+\!n_{k,v})
\>
for $v=1,..,V$.
This distribution is discrete, so $\m\Phi$ becomes a sparse matrix.
The approximation is accurate even for small values of $n_{k,v}$, and \textcite{terenin19b} proves that the approximation error will vanish for large data sets in the sense of convergence in distribution.

If $\v\beta$ is uniform, we can further use sparsity to accelerate sampling $\varphi_{k,v}$.
Since a sum of Poisson random variables is Poisson, we can split $\varphi_{k,v} = \varphi_{k,v}^{\smash{(\beta)}} + \varphi_{k,v}^{\smash{(\m{n})}}$.
We then sample $\varphi_{k,v}^{\smash{(\beta)}}$ sparsely by introducing a Poisson process and sampling its points uniformly, and sample $\varphi_{k,v}^{\smash{(\m{n})}}$ sparsely by iterating over nonzero entries of $\m{n}$.

\begin{figure*}[b!]
\vspace*{-1ex}
\footnoterule
\footnotesize
\quad \footnotemark[1]See \url{http://mallet.cs.umass.edu} and \url{https://github.com/lejon/PartiallyCollapsedLDA}. 
AP and CGCBIB can be found therein. NeurIPS and PubMed can be found at \url{https://archive.ics.uci.edu/ml/datasets/bag+of+words}.
Full output of experiments can be found at \url{https://github.com/aterenin/Parallel-HDP-Experiments/}.
\label{ftn:code}
\end{figure*}

For $\v{z}$, the full conditional
\<
\P(z_{i,d} = k &\given \v{z}_{-i,d},\m\Phi,\v\Psi)
\\
&\propto \phi_{k,v(i)}\sbr{\alpha \Psi_k + m^{-i}_{d,k}}
\\
&\propto \underbracket[0.1ex]{\phi_{k,v(i)}\alpha \Psi_k \vphantom{\phi_{k,v(i)}m^{-i}_{d,k}} }_{(a)\vphantom{(b)}} + \underbracket[0.1ex]{\phi_{k,v(i)}m^{-i}_{d,k} \vphantom{\phi_{k,v(i)}\alpha \Psi_k} }_{(b)\vphantom{(a)}}
\>
is similar to to the one in partially collapsed LDA \cite{magnusson18}---the difference is the presence of $\Psi_k$.
As $\Psi_k$ only enters the expression through component $(a)$ and is identical for all $z_{i,d}$, it can be absorbed at each iteration directly into an alias table \cite{walker77, li14}.
Component $(b)$ can be computed efficiently by utilizing sparsity of $\m\Phi$ and $\m{m}$ and iterating over whichever has fewer non-zero entries.

\subsection{Direct sampling of $\vec{l}$}
\label{sec:bin}

Rather than sampling $\v{b}$, whose size will grow linearly with the number of documents, we introduce a scheme for sampling the sufficient statistic $\v{l}$ directly.
Observe that
\[
l_k = \sum_{d=1}^D \sum_{\substack{i=1\\b_{i,d} = 1}}^{N_d} \1_{z_{n,d}=k} = \sum_{d=1}^D \sum_{\substack{i=1\\z_{i,d} = k}}^{N_d} \1_{b_{i,d}=1}
\]
where the domain of summation and the value of the indicators have been switched.
By definition of $b_{i,d}$, we have
\[
\sum_{\substack{i=1\\z_{i,d} = k}}^{N_d} \1_{b_{i,d}=1} = \sum_{j=1}^{m_{d,k}} b_{j,d,k}
\]
where
\[
b_{j,d,k} \~[Ber] \del{\frac{\Psi_k \alpha}{\Psi_k \alpha + j - 1}}
.
\]
Summing this expression over documents, we obtain the expression
\<
\label{eqn:bin}
\!\!\!l_k &\!=\! \sum_{j=1}^{\mathclap{\max_d m_{d,k}}} c_{j,k}
&
c_{j,k} &\!\~\!\f{Bin}\del{\!D_{k,j}, \frac{\Psi_k \alpha}{\Psi_k \alpha\!+\!j\!-\!1}\!}\!\!
\>
where $D_{k,j}$ is the total number of documents with $m_{d,k} \geq j$.
Since $m_{d,k} = 0$ for all topics $k$ without any tokens assigned, we only need to sample $\v{l}$ for topics that have tokens assigned to them.
This idea can also be straightforwardly applied to other HDP samplers \cite{chang14,ge15}, by allowing one to derive alternative full conditionals in lieu of the \emph{Stirling distribution} \cite{antoniak74}.
The complexity of sampling $\v{l}$ directly is constant with respect to the number of documents, and depends instead on the maximum number of tokens per document.

To handle the bookkeeping necessary for computing $D_{k,j}$, we introduce a sparse matrix $\m{d}$ of size $K \x \max_d N_d$ whose entries $d_{k,p}$ are the number of documents for topic $k$ that have a total of $p$ topic indicators assigned to them.
We increment $\m{d}$ once $\v{z}_d$ been sampled by iterating over non-zero elements in $\v{m}_d$.
We then compute $D_{k,j}$ as the reverse cumulative sum of the rows of $\m{d}$.

\subsection{Poisson P\'{o}lya urn partially collapsed Gibbs sampling}

Putting all of these ideas together, we obtain the following algorithm.

\begin{Algorithm} \label{alg:ppu-hdp}
Repeat until convergence.
\1* Sample $\v{\phi}_k \~[PPU](\v{n}_k + \v{\beta})$ in parallel over topics for $k=1,..,K^*$.
\2* Sample $z_{i,d} \propto \phi_{k,v(i)}\,\alpha\,\Psi_k + \phi_{k,v(i)}\, m_{d,k}^{-i}$ in parallel over documents for $d=1,..,D$.
\3* Sample $l_k$ according to equation \eqref{eqn:bin} in parallel over topics for $k=1,..,K^*$.
\4* Sample $\v\Psi$ according to equations \eqref{eqn:post-gem-1}--\eqref{eqn:post-gem-2}, except with $\varsigma_{K^*} = 1$.
\0*
\end{Algorithm}

\noindent
Algorithm \ref{alg:ppu-hdp} is sparse, massively parallel, defined on a fixed finite state space, and contains no infinite computations in any of its steps.
The Gibbs step for $\m\Phi$ converges in distribution \cite{terenin19b} to the true Gibbs steps as $N \-> \infty$, and the Gibbs step for $\v\Psi$ converges almost surely \cite{ishwaran01} to the true Gibbs step as $K^*\->\infty$.

\subsection{Computational complexity}

We now examine the per-iteration computational complexity of Algorithm \ref{alg:ppu-hdp}.
To proceed, we fix $K^*$ and maximum document size $\max_d N_d$, and relate the vocabulary size $V$ with the number $N$ of total words as follows.

\begin{assumption*}[Heaps' Law]
The number of unique words in a corpus follows Heaps' law \cite{heaps78} $V = \xi N^\zeta$ with constants $\xi > 0$ and $\zeta < 1$.
\end{assumption*}

The per-iteration complexity of Algorithm \ref{alg:ppu-hdp} is equal to the sum of the per-iteration complexity of sampling its components.
The sampling complexities of $\v\Psi$ and $\v{l}$ are constant with respect to the number of tokens, and the sampling complexity of $\m\Phi$ has been shown by \textcite{magnusson18} to be negligible under the given assumptions.
Thus, it suffices to consider $\v{z}$.

At a given iteration, let $\smash{K_{d(i)}^{(\m{m})}}$ be the number of existing topics in document $d$ associated with word token $i$, and let $\smash{K^{(\m\Phi)}_{v(i)}}$ be the number of nonzero topics in the row of $\m\Phi$ corresponding to word token $i$.
It follows immediately from the argument given by \textcite{terenin19b} that the per-iteration complexity of sampling each topic indicator $z_i$ is
\[
\c{O}\sbr{\min\del{K^{(\m{m})}_{d(i)}, K^{(\m\Phi)}_{v(i)}}}
.
\]
Algorithm \ref{alg:ppu-hdp} is thus a doubly sparse algorithm.

\begin{table*}[b!]
\begin{center}
\captionfont
\begin{tabular}{l r r r c r r r}
\hline
Corpus & $V$ & $D$ & $N$ && Iterations & Threads & Runtime\\
\cline{1-4} \cline{6-8}
AP & 7 074 & 2 206 & 393 567 && 100 000 & 8 & 3.8 hours\\
CGCBIB & 6 079 & 5 940 & 570 370 && 100 000 & 12 & 2.7 hours\\
NeurIPS & 12 419 & 1 499 & 1 894 051 && 255 500 & 8 & 24 hours \\
PubMed & 89 987 & 8 199 999 & 768 434 972 && 25 000 & 20 & 82.4 hours\\
\hline
\end{tabular}
\end{center}
\caption{Corpora used in experiments, together with compute configuration.} \label{tbl:corpora-runtime}
\end{table*}

\section{Performance results} \label{sec:results}

\begin{figure*}
\begin{center}
\includegraphics{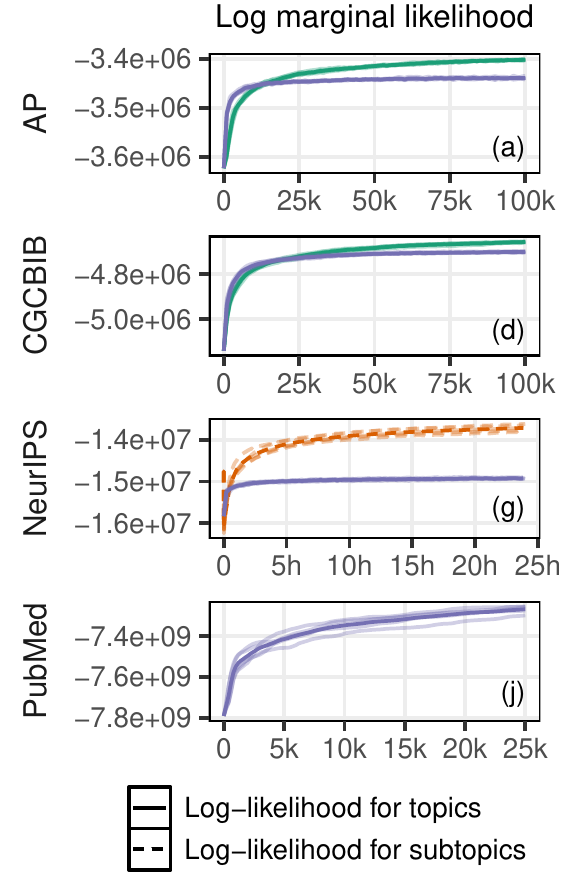}
\includegraphics{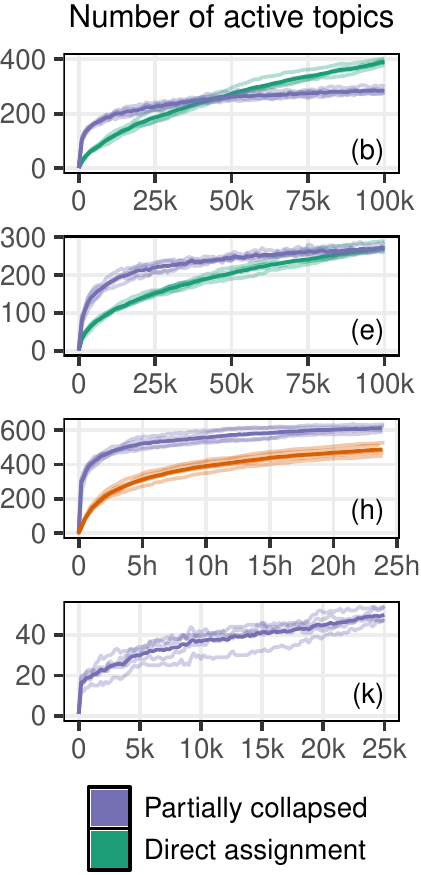}
\includegraphics{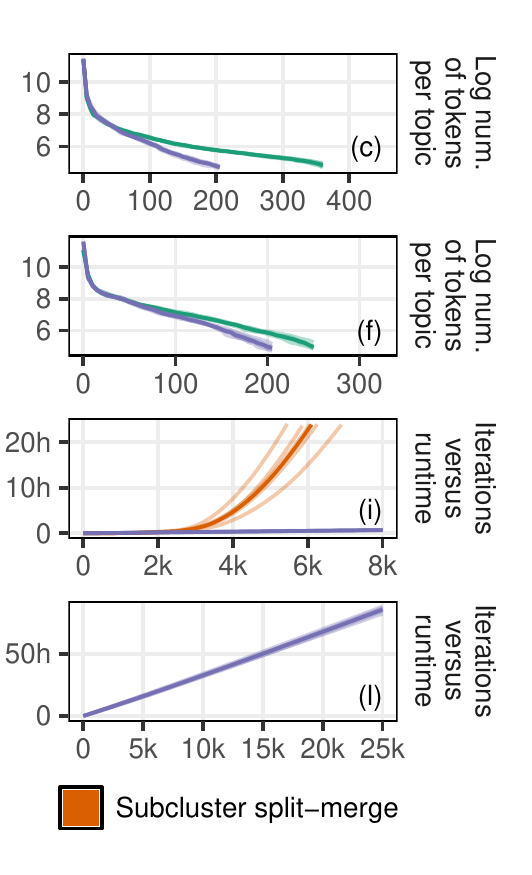}
\end{center}
\caption{Trace plots for log-likelihood, number of active topics, and additional metrics for CGCBIB, NeurIPS, and PubMed.
On the $x$ axis,  per-iteration scale is used for AP, CGCBIB and PubMed, and real-time scale is used for NeurIPS.
Algorithms used are partially collapsed HDP for all corpora, direct assignment HDP for AP and CGCBIB, and subcluster split-merge HDP for NeurIPS.
Individual traces are partially transparent, and their mean is opaque.} \label{fig:trace}
\end{figure*}

To study performance of the \emph{partially collapsed} sampler---Algorithm \ref{alg:ppu-hdp}---we implemented it in Java using the open-source \textsc{Mallet}\hyperref[ftn:code]{\footnotemark[1]} \cite{mccallum02} topic modeling framework.
We ran it on the \textsc{AP}, \textsc{CGCBIB}, \textsc{NeurIPS}, and \textsc{PubMed} corpora,\hyperref[ftn:code]{\footnotemark[1]} which are summarized in Table \ref{tbl:corpora-runtime}.
Prior hyperparameters controlling the degree of sparsity were set to $\alpha=0.1, \beta=0.01, \gamma=1$.
We set $K^* = 1000$ and observed no tokens ever allocated to the topic $K^*$.
Data were preprocessed with default Mallet \cite{mccallum02} stop-word removal, minimum document size of 10, and a rare word limit of 10.
Following \textcite{teh06a}, the algorithm was initialized with one topic.
All experiments were repeated five times to assess variability.
Total runtime for each experiment is given in Table \ref{tbl:corpora-runtime}.

To assess Algorithm \ref{alg:ppu-hdp} in a small-scale setting, we compare it to the widely-studied sparse fully collapsed \emph{direct assignment} sampler of \textcite{teh06a}, which is not parallel.
We ran 100 000 iterations of both methods on AP and CGCBIB.
We selected these corpora because they were among the larger corpora on which it was feasible to run our direct assignment reference implementation within one week.

Trace plots for the log marginal likelihood for $\v{z}$ given $\v\Psi$ and the number of active topics, i.e., those topics assigned at least one token, can be seen in Figure \ref{fig:trace}(a,d) and Figure \ref{fig:trace}(b,e), respectively.
The direct assignment algorithm converges slower, but achieves a slightly better local optimum in terms of marginal log-likelihood, compared to our method.
This fact indicates that the direct assignment method may stabilize around a different local optimum, and may represent a potential limitation of the partially collapsed sampler in settings where non-parallel methods are practical.

To better understand the distributional differences between the algorithms, we examined the number of tokens per topic, which can be seen in Figure \ref{fig:trace}(c,f).
The partially collapsed sampler is seen to assign more tokens to smaller topics, indicating that it stabilizes around a local optimum with slightly broader semantic themes.

To visualize the effect this has on the topics, we examined the most common words for each topic.
Since the algorithms generate too many topics to make full examination practical, we instead compute a quantile summary with five topics per quantile.
The quantile is computed by ranking all topics by the number of tokens, choosing the five closest topics to the $100\%$, $75\%$, $50\%$, $25\%$, and $5\%$ quantiles in the ranking, and computing their top words.
This approach gives a representative view of the algorithm's output for large, medium, and small topics.
Results may be seen in Appendix \ref{apdx:cgcbib} and Appendix \ref{apdx:ap}---we find the direct assignment and partially collapsed samplers to be mostly comparable, with substantial overlap in top words for common topics.

Next, we assess Algorithm \ref{alg:ppu-hdp} in a more demanding setting and compare against previous parallel state-of-the-art.
There are various scalable samplers available for the HDP. For a fair comparison, we restrict ourselves to those samplers designed for topic models and explicitly incorporate sparsity of natural language in their construction.
Among these, we selected the parallel \emph{subcluster split-merge} algorithm of \textcite{chang14} as our baseline because it was used in the largest-scale benchmark of the HDP topic model performed to date to our awareness, and shows comparable performance to other methods \cite{ge15}.
The subcluster split-merge algorithm is designed to converge with fewer iterations, but is more costly to run per iteration.
Thus, we used a fixed computational budget of 24 hours of wall-clock time for both algorithms.
Computation was performed on a system with a 4-core 8-thread CPU and 8GB RAM.

Results can be seen in Figure \ref{fig:trace}(g)---note that the subcluster split-merge algorithm is parametrized using \emph{sub-topic indicators} and \emph{sub-topic probabilities}, so its numerical log-likelihood values are not directly comparable to ours and should be \emph{interpreted purely to assess convergence}.
Algorithm \ref{alg:ppu-hdp} stabilizes much faster with respect to both the number of active topics in Figure \ref{fig:trace}(g), and marginal log-likelihood in Figure \ref{fig:trace}(h).
The subcluster split-merge algorithm adds new topics one-at-a-time, whereas our algorithm can create multiple new topics per iteration---we hypothesize this difference leads to faster convergence for Algorithm \ref{alg:ppu-hdp}.

In Figure \ref{fig:trace}(i), we observe that the amount of computing time per iteration increases substantially for the subcluster split-merge method as it adds more topics.
For Algorithm \ref{alg:ppu-hdp}, this stays approximately constant for its entire runtime.

To evaluate the topics produced by the algorithms, we again examined the most common words for each topic via a quantile summary, given in Appendix \ref{apdx:neurips}.
We find the subcluster split-merge algorithm appears to generate topics with slightly more semantic overlap compared to Algorithm \ref{alg:ppu-hdp}, but otherwise produces comparable output.

Finally, to assess scalability, we ran 25 000 iterations of Algorithm \ref{alg:ppu-hdp} on PubMed, which contains 768m tokens.
To our knowledge, this dataset is an order of magnitude larger than any datasets used in previous MCMC-based approaches for the HDP.
Computation was performed on a compute node with 2x10-core CPUs with 20 threads and 64GB of RAM.
The marginal likelihood and number of active topics are given in Figure \ref{fig:trace}(j) and Figure~\ref{fig:trace}(k).

To evaluate the topics discovered by the algorithm, we examined their most common words---these may be seen in full in Appendix \ref{apdx:pubmed}.
We observe that the semantic themes present in the topics vary according to how many tokens they have: topics with more tokens appear to be broader, whereas topics with fewer tokens appear to be more specific.
This behavior illustrates a key difference between the HDP and methods like LDA, which do not contain a learned global topic distribution $\v\Psi$ in their formulation. 
We suspect the effect is particularly pronounced on PubMed compared to CGCBIB and NeurIPS due to its large number of tokens.

\begin{figure*}[t!]
\begin{center}
\captionfont
\renewcommand{\arraystretch}{0.95}
\begin{tabular}{c C{0.145\textwidth} C{0.145\textwidth} C{0.145\textwidth} C{0.145\textwidth} C{0.145\textwidth}}
\cline{2-6}
$k$ & Topic 1 & Topic 5 & Topic 9 & Topic 13 & Topic 17
\\
$n_{k,\.}$ & 42 395 289 & 23 907 517 & 22 167 377 & 20 925 933 & 18 924 590
\\
\cline{2-6}
& care & cancer & protein & protein & cell \\
& health & tumor & binding & cell & neuron \\
& patient & patient & membrane & kinase & electron \\
\smash{\rotatebox[origin=c]{90}{\parbox{20ex}{\centering }}} & medical & cell & acid & expression & brain \\
& research & carcinoma & activity & receptor & rat \\
& system & breast & cell & activation & nerve \\
& clinical & tumour & gel & pathway & fiber \\
& cost & survival & human & phosphorylati & nucleus \\
\\
\end{tabular}
\begin{tabular}{c C{0.145\textwidth} C{0.145\textwidth} C{0.145\textwidth} C{0.145\textwidth} C{0.145\textwidth}}
\cline{2-6}
$k$ & Topic 21 & Topic 25 & Topic 29 & Topic 33 & Topic 37
\\
$n_{k,\.}$ & 18 033 777 & 16 308 024 & 15 128 822 & 13 562 338 & 10 819 160
\\
\cline{2-6}
& cell & rat & gene & infection & plant \\
& growth & day & mutation & strain & strain \\
& expression & mice & genetic & antibiotic & acid \\
\smash{\rotatebox[origin=c]{90}{\parbox{20ex}{\centering }}} & factor & liver & chromosome & bacterial & growth \\
& beta & animal & analysis & isolates & extract \\
& human & effect & genes & bacteria & activity \\
& mrna & control & polymorphism & resistance & cell \\
& endothelial & mg & dna & coli & production \\
\\
\end{tabular}

\vspace*{-2ex}
\end{center}
\caption{Top 8 words for topics obtained by Algorithm \ref{alg:ppu-hdp} on PubMed, together with topic index $k$ and total number of words $n_{k,\.}$ present in the topic. 
We observe that the topics range from broad to specific: this is a consequence of the hierarchical Dirichlet process prior via the inclusion of the global topic proportions $\v\Psi$. 
Topics obtained by Algorithm \ref{alg:ppu-hdp} on all corpora may be seen in Appendix \ref{apdx:ap}, Appendix \ref{apdx:cgcbib}, Appendix \ref{apdx:neurips}, and Appendix \ref{apdx:pubmed}.
}
\end{figure*}

\section{Discussion} \label{sec:discussion}

In this work, we introduce the parallel partially collapsed Gibbs sampler---Algorithm \ref{alg:pc-hdp}---for the HDP topic model, which converges to the correct target distribution.
We propose a doubly sparse approximate sampler---Algorithm \ref{alg:ppu-hdp}---which allows the HDP to be implemented with per-token sampling complexity of $\smash{\c{O}\sbr[1]{\min\del[1]{K^{(\m{m})}_{d(i)}, K^{(\m\Phi)}_{v(i)}}}}$ which is the same as that of P\'{o}lya Urn LDA \cite{terenin19b}.
Compared to other approaches for the HDP, it offers the following improvements.

\1 The algorithm is fully parallel in all steps.
\2 The topic indicators $\v{z}$ utilize all available sources of sparsity to accelerate sampling.
\3 All steps not involving $\v{z}$ have constant complexity with respect to data size.
\4 The proposed sparse approximate algorithm becomes exact as $N\->\infty$ and $K^*\->\infty$.
\0
These improvements allow us to train the HDP on larger corpora.
The data-parallel nature of our approach means that the amount of available parallelism increases with data size.
This parallelism avoids load-balancing-related scalability limitations pointed out by \textcite{gal14}.

Nonparametric topic models are less straightforward to evaluate empirically than ordinary topic models.
In particular, we found topic coherence scores \cite{mimno11} to be strongly affected by the number of active topics $K$, which causes preference for models with fewer topics and more semantic overlap per topic.
We view the development of summary statistics that are $K$-agnostic and those measuring other aspects of topic quality such as overlap, to be an important direction for future work.
We are particularly interested in techniques that can be used to compare algorithms for sampling from the same model defined over fully disjoint state spaces, such as Algorithm \ref{alg:ppu-hdp} and the subcluster split-merge algorithm in Section \ref{sec:results}.

Partially collapsed HDP can stabilize around a different local mode than fully collapsed HDP as proposed by \textcite{teh06a}.
There have been attempts to improve mixing in that sampler \cite{chang14}, including the use of Metropolis-Hastings steps for jumping between modes \cite{jain04}.
These techniques are largely complementary to ours and can be explored in combination with the ideas presented here.

The HDP posterior is a heavily multimodal target for which full posterior exploration is known to be difficult \cite{chang14,gal14,buntine14}, and sampling schemes are generally used more in the spirit of optimization than traditional MCMC.
These issues are mirrored in other approaches, such as variational inference. 
There, restrictive mean-field factorization assumptions are often required, which reduces the quality of discovered topics.
We view MAP-based analogs of ideas presented here as a promising direction, since these may allow additional flexibility that may enable faster training.

Many of the ideas in this work, such as the binomial trick, are generic and apply to any topic model structurally similar to the HDP's GEM representation \cite{teh06a} given in Section \ref{sec:alg}.
For example, one could consider an informative prior for $\v\Psi$ in lieu of $\f{GEM}(\gamma)$, potentially improving convergence and topic quality, or developing parallel schemes for other nonparametric topic models such as Pitman-Yor models \cite{teh06b}, tree-based models \cite{hu12,paisley15}, embedded topic models \cite{dieng19}, as well as nonparametric topic models used within data-efficient language models \cite{guo19} in future work.

\paragraph{Conclusion}
We introduce the doubly sparse partially collapsed Gibbs sampler for the hierarchical Dirichlet process topic model.
By formulating this algorithm using a representation of the HDP which connects it with the well-studied Latent Dirichlet Allocation model, we obtain a parallel algorithm whose per-token sampling complexity is the minima of two sparsity terms.
The ideas used apply to a large array of topic models, for example, dynamic topic models with $\m\Phi$ time-varying, which possess the same full conditional for $\v{z}$.
Our algorithm for the HDP scales to a 768m-token corpus (PubMed) on a single multicore machine in under four days.

The proposed techniques leverage parallelism and sparsity to scale nonparametric topic models to larger datasets than previously considered feasible for MCMC or other methods possessing similar convergence properties.
We hope these contributions enable wider use of Bayesian nonparametrics for large collections of text.

\paragraph{Acknowledgments}

The research was funded by the Academy of Finland (grants 298742, 313122), as well as the Swedish Research Council (grants 201805170, 201806063).
Computations were performed using compute resources within the Aalto University School of Science and Department of Computing at Imperial College London. We also acknowledge the support of Ericsson AB.

\bibliographystyle{acl_natbib}
\bibliography{references}

\onecolumn
\appendix

\section{Appendix: sufficiency of $\vec{l}$ and full conditional for $\vec{b}$}
\label{apdx:sufficiency}

Recall that the one-step-ahead conditional probability mass function in a P\'{o}lya sequence taking values in $\N$ with concentration parameter $\alpha$ and base probability mass function $\v\Psi$ is
\[
p(z_i \given z_{i-1},..,z_1, \v\Psi) = \sum_{j=1}^{i-1} \frac{1}{i - 1 + \alpha} \1_{z_j}(z_i) + \frac{\alpha}{i-1+\alpha} \Psi_{z_i}
.
\]
Introducing the random variable
\[
b_i \~[Ber]\del{\frac{\alpha}{i-1+\alpha}}
\]
we can express the one-step-ahead conditional distribution as
\[
p(z_i \given z_{i-1},..,z_1, b_i, \v\Psi) = \1_{b_i=0} \sum_{j=1}^{i-1} \frac{1}{i - 1} \1_{z_j}(z_i) + \1_{b_i = 1} \Psi_{z_i}
.
\]
The joint probability mass function for $\v{z} \given \v{b},\v\Psi$ is then
\[
p(\v{z}\given\v{b},\v\Psi) = \prod_{i=1}^N p(z_i \given z_{i-1},..,z_1, \v{b},\v\Psi) = \prod_{i=1}^N \sbr{\sum_{j=1}^{i-1} \1_{b_i=0} \1_{z_j}(z_i) + \1_{b_i=1} \Psi_{z_i}}
.
\]
Note that $\1_{b_i = 0} = 1 \iff \1_{b_i=1} = 0$ and vice versa.
Thus each term in the product for $\v{z}\given\v{b},\v\Psi$ only has one component, and we may express $\v{z} \given \v{b},\v\Psi$ as
\[
p(\v{z} \given \v{b}, \v\Psi) = \underbracket[0.1ex]{\prod_{\substack{i = 1 \\ b_i \neq 1}}^N \sum_{j=1}^{i-1} \frac{1}{i-1} \1_{z_j}(z_i)}_{\t{doesn't enter posterior}}  \prod_{\substack{i = 1 \\ b_i = 1}}^N \prod_{k=1}^\infty \Psi_k^{\1_k(z_i)}
\]
where we have re-expressed the probability mass function of $\v\Psi$ in a form that emphasizes conjugacy.
Thus for any prior, the posterior will only depend on the likelihood of the values of $z_i$ for which $b_i = 1$.
The sufficient statistic is
\[
l_k = \sum_{\substack{i=1\\b_i = 1}}^N \1_{z_i = k}
.
\]
Next, for a given $i' \in \{1,..,N\}$, we can calculate the posterior of a component $b_{i'}$ as
\<
\P(b_{i'} = 1 \given \v{z},\v\Psi, \v{b}_{-i'}) &\propto \del{\frac{\alpha}{i'-1+\alpha}} \prod_{\substack{i = 1 \\ b_i \neq 1}}^N \sum_{j=1}^{i-1} \frac{1}{i-1} \1_{z_j}(z_i)  \prod_{\substack{i = 1 \\ b_i = 1}}^N \Psi_{z_i}
\\
&\propto \alpha\Psi_{z_{i'}}
\\
\P(b_{i'} = 0 \given \v{z},\v\Psi, \v{b}_{-i'}) &\propto \del{\frac{i'-1}{i'-1+\alpha}} \prod_{\substack{i = 1 \\ b_i \neq 1}}^N \sum_{j=1}^{i-1} \frac{1}{i-1} \1_{z_j}(z_i)  \prod_{\substack{i = 1 \\ b_i = 1}}^N \Psi_{z_i}
\\
&\propto \sum_{i=1}^{i'-1} \1_{z_i}(z_i')
\>
where we have divided both expressions by
\[
\frac{1}{i'-1+\alpha} \prod_{\substack{i = 1 \\ b_i \neq 1\\ i \neq i'}}^N \sum_{j=1}^{i-1} \frac{1}{i-1} \1_{z_j}(z_i) \prod_{\substack{i = 1 \\ b_i = 1\\ i \neq i'}}^N \Psi_{z_i}
\]
which is constant with respect to $b_{i'}$.
Note that full conditionally, we have $b_i \indep b_{i'}$ for $i\neq i'$.
This gives the desired expressions and concludes the derivation.

\section{Appendix: full conditional for $\mPsi$}
\label{apdx:gem}

Before proceeding with the derivation, we first comment on Proposition \ref{prop:gem} and differences between the GEM distribution and Dirichlet process, which otherwise appear superficially similar.
The GEM distribution $\Psi^{\f{GEM}} \~[GEM](\gamma)$ is defined as
\<
\Psi_k^{\f{GEM}} &= \varsigma_k \prod_{i=1}^{k-1} (1 - \varsigma_i)
&
\varsigma_k^{\f{GEM}} &\~[Beta](1,\gamma)
.
\>
On the other hand, a Dirichlet process $\Psi^{\f{DP}}\~[DP](\gamma,F)$ is defined as 
\<
\v\Psi^{\f{DP}} &= \sum_{k=1}^\infty \pi_k \delta_{\vartheta_k}
&
\vartheta_k &\~ F
&
\pi_k &= \varsigma_k \prod_{i=1}^{k-1} (1 - \varsigma_i)
&
\varsigma_k &\~[Beta](1,\gamma)
.
\>
From a Bayesian perspective, this extra stage---the presence of $\vartheta_k$---prevents one from applying standard results on conjugacy of Dirichlet processes.
The joint distribution of a finite set of states $(\Psi_{k_1}^{\f{GEM}},..,\Psi_{k_K}^{\f{GEM}})$ does not admit a closed-form expression, so we seek to derive the posterior conditional in a different way.

Rather than proving conjugacy for $(\Psi_{k_1}^{\f{GEM}},..,\Psi_{k_K}^{\f{GEM}})$ directly, we look for a larger finite-dimensional distribution within which $(\Psi_{k_1}^{\f{GEM}},..,\Psi_{k_K}^{\f{GEM}})$ sits that has better conjugacy properties.
The \emph{generalized Dirichlet} distribution of \textcite{connor69} fulfills this criteria. 
The conjugacy relationship we seek follows from the general property that conditioning and marginalization commute.
This will be shown to yield the posterior
\<
\Psi_k^{\f{GEM}} &= \varsigma_k \prod_{i=1}^{k-1} (1 - \varsigma_i)
&
\varsigma_k &\~[Beta](a^{(\v\Psi)}_k,b^{(\v\Psi)}_k)
&
a^{(\v\Psi)}_k &= 1 + l_k
&
b^{(\v\Psi)}_k &= \gamma + \sum_{i=k+1}^{\infty} l_i
.
\>
For comparison, a posterior Dirichlet process is given by
\<
\v\Psi^{\f{DP}} &= \sum_{k=1}^\infty \pi_k \delta_{\vartheta_k}
&
\vartheta_k &\~ \frac{n}{\alpha+n}\v{l} + \frac{\alpha}{\alpha+n}F
&
\pi_k &= \varsigma_k \prod_{i=1}^{k-1} (1 - \varsigma_i)
&
\varsigma_k &\~[Beta](1,\gamma + n)
\>
which shows that this relatively mild difference in the prior yields a posterior of a rather different form.

We now proceed to formally calculate this posterior distribution, starting from a GEM prior and discrete likelihood.
Since we are working in a nonparametric setting, we begin by introducing the necessary formalism.
We then introduce our finite-dimensional approximating prior and compute the posterior under it. For this, we use commutativity of conditioning and marginalization to deduce the full infinite-dimensional posterior.

\begin{definition}[Preliminaries]
Let $(\Omega,\s{F},\P)$ be a probability space.
Let $\c{M}_s(\N)$ be the space of signed measures, equipped with the topology of weak convergence.
Let $\c{M}_1(\N) \subset \c{M}_s(\N)$ be the space of probability measures over $\N$, and identify $\c{M}_1(\N)$ with the probability simplex by the homeomorphism $\c{M}_1(\N) \isom \{\v{x} \in \ell^1 : \forall i, x_i > 0, \sum_{i=1}^\infty x_i = 1\}$.
Let $N\in\N$, let $\v{x} \in \N^N$, and let $\v{l} \in N^N$ be its empirical counts, defined by $\v{l} = \sum_{i=1}^N \v{1}_{x_i}$ where $\v{1}_{x_i}$ is equal to $1$ for coordinate $x_i$ and $0$ for all other coordinate.
Let $\gamma \in \R^+$.
Recall that $\N^N$ and $\c{M}_1(\N)$, endowed with the discrete topology and topology of weak convergence, respectively, are both Polish spaces---hence, the Disintegration Theorem (\textcite{ambrosio07b}, Theorem 5.3.1; \textcite{bogachev07b}, Corollary 10.4.15) holds in both spaces.
We associate each random variable $y:\Omega\-> Y$ with its pushforward probability measure $\pi_y(A_y) = \sbr{y_*\P}(A_y) = \P[y^{-1}(A_y)]$, and each conditional random variables $\theta\given y:\Omega\x Y \-> \Theta$ with its pushforward regular conditional probability measure $\pi_{y\given\theta}(A_y \given \theta) = \sbr{(y\given\theta)_*\P}(A_y) = \P[(y\given\theta)^{-1}(A_y)]$, where the preimage is taken with respect to $y$.
\end{definition}

\begin{definition}[Discrete likelihood]
For all $\v\Psi \in \c{M}_1(\N)$, define the conditional random variable $\v{x}\given\v\Psi : \Omega \x \c{M}_1(\N)\-> \N^N$ by its probability mass function
\[
p(\v{x}\given\v\Psi) = \prod_{i=1}^N \prod_{k=1}^\infty \Psi_k^{\1_k(x_i)}
.
\]
We say $\v{x}\given\v\Psi \~[Discrete](\v\Psi)$.
\end{definition}

\begin{definition}[GEM]
Let $\v\Psi:\Omega\-> \c{M}_1(\N)$ be a random variable defined by
\<
\Psi_k &= \varsigma_k \prod_{i=1}^{k-1} (1 - \varsigma_i)
&
\varsigma_k &\~[Beta](1,\gamma)
.
\>
We say $\v\Psi\~[GEM](\gamma)$.
\end{definition}

\begin{definition}[Finite GEM]
Let $\v\Psi:\Omega\-> \c{M}_1(\N)$ be a random variable defined by
\<
\Psi_k &= \varsigma_k \prod_{i=1}^{k-1} (1 - \varsigma_i)
&
\varsigma_k &\~[Beta](1,\gamma)
&
\varsigma_K &= 1
.
\>
We say $\v\Psi\~[FGEM](\gamma,K)$.
\end{definition}

\begin{definition}[Posterior]
Let $\v\Psi\given\v{x}$ be the unique conditional random variable given by the Disintegration Theorem, where uniqueness follows from almost sure uniqueness by virtue of the marginal measure $\pi_{\v{x}}(\cdot) = \int_{\c{M}_1(\N)}\pi_{\v{x}\given\v\Psi}(\cdot\given\v\Psi) \d\pi_{\Psi}$ being absolutely continuous with respect to the counting measure on $\N^N$, which has no non-empty null sets.
\end{definition}

\begin{result} \label{res:gd}
Let $\v{x}\given\v\Psi \~[Discrete](\v\Psi)$.
Let $\v{x}\in\N^N$, and let $K > \sup \v{x}$.
Let $\v\Psi\~[FGEM](\gamma,K)$.
Then for any $\v{x}$ with empirical counts $\v{l}$, we have that $\v\Psi\given\v{x} : \Omega \x \N^N \-> \c{M}_1(\N)$ is a conditional random variable defined by
\<
\Psi_k &= \varsigma_k \prod_{i=1}^{k-1} (1 - \varsigma_i)
&
\varsigma_k &\~[Beta](a^{(\v\Psi)}_k,b^{(\v\Psi)}_k)
&
\varsigma_K &= 1
\>
where
\<
a^{(\v\Psi)}_k &= 1 + l_k
&
b^{(\v\Psi)}_k &= \gamma + \sum_{i=k+1}^K l_i
.
\>
\end{result}

\begin{proof}
It is shown by \textcite{connor69} that $\v\Psi\~[FGEM](\gamma,K)$ is a special case of the \emph{generalized Dirichlet} distribution, which admits a general stick-breaking representation.
Thus, its probability density function is
\[
f(\v\Psi) \propto \Psi_K^{\gamma - 1} \prod_{k=1}^{K-1} \sbr{\sum_{k'=1}^K \Psi_{k'}}^{-1}
\]
which we have expressed in a simplified form.
By conjugacy, for a given $\v{x}$ and associated $\v{l}$ the posterior probability density is
\[
f(\v\Psi\given\v{x}) \propto \Psi_K^{(\gamma + l_K) - 1} \prod_{k=1}^{K-1}\sbr{\Psi_k^{(1+l_k)-1} \sbr{\sum_{k'=k}^K \Psi_{k'}}^{\gamma + \sum_{i=k}^K l_i - \sbr{(1 + l_k) + \gamma + \sum_{i=k+1}^K l_i}}}
\]
which is again a generalized Dirichlet admitting the necessary stick-breaking representation, which we have expressed in a form that emphasizes its posterior hyperparameters.
\end{proof}

\begin{remark}
It is now clear that the assumption $\v{x}\given\v\Psi \~[Discrete](\v\Psi)$ is indeed taken without loss of generality, because if we instead took $\v{x}\given\v\Psi$ to be given by a P\'{o}lya sequence, then by sufficiency the prior-to-posterior map would be identical.
\end{remark}

{
\renewcommand{\thetheorem}{1}
\begin{proposition}
Without loss of generality, suppose
\<
\v\Psi &\~[GEM](\gamma)
&
\v{x}\given\v\Psi &\~[Discrete](\v\Psi)
.
\>
Then $\v\Psi\given\v{x}$ is given by
\<
\label{eqn:posterior-gem}
\Psi_k &= \varsigma_k \prod_{i=1}^{k-1} (1 - \varsigma_i)
&
\varsigma_k &\~[Beta](a^{(\v\Psi)}_k,b^{(\v\Psi)}_k)
&
a^{(\v\Psi)}_k &= 1 + l_k
&
b^{(\v\Psi)}_k &= \gamma + \sum_{i=k+1}^{\infty} l_i
\>
where $\v{l}$ are the empirical counts of $\v{x}$.
\end{proposition}
}

\begin{proof}
Let $I \subset \N$ be an arbitrary finite index set, and let $\v\Psi_I \given \v{x}$ be the finite-dimensional marginal projection of $\v\Psi\given\v{x}$ onto the coordinates contained in $I$.
Let $K > \sup I$, let $\v\Psi^{(K)}\given\v{x}$ be the posterior conditional random variable under $\v\Psi^{(K)}\~[FGEM](\gamma,K)$, and let $\v\Psi^{(K)}_I\given\v{x}$ be the marginal consisting of those coordinates contained in $I$.
By construction, $\v\Psi^{(K)}_I\given\v{x}$ equals $\v\Psi_I\given\v{x}$ in distribution.
Since by the Disintegration Theorem, conditioning and marginalization commute, the set $I$ is arbitrary, and $\v\Psi\given\v{x}$ is uniquely determined by its finite-dimensional marginal projections, the claim follows.
\end{proof}

\newpage
\section{Appendix: quantile summary of topics for \textsc{AP}}
\label{apdx:ap}

Here we display a multi-quantile summary for AP, obtained by ranking all topics with at least 100 tokens by their total number of tokens, computing the $\varpi=100\%$, $75\%$, $50\%$, $25\%$, and $5\%$ quantiles.
We compute the five topics closest to each quantile by number of tokens, and display their top-eight words.

\begin{center}
\captionfont
\renewcommand{\arraystretch}{0.95}
\begin{tabular}{c C{0.145\textwidth} C{0.145\textwidth} C{0.145\textwidth} C{0.145\textwidth} C{0.145\textwidth}}
\cline{2-6}
$k$ & Topic 1 & Topic 2 & Topic 3 & Topic 4 & Topic 5
\\
$n_{k,\.}$ & 93 207 & 57 249 & 15 874 & 13 360 & 10 176
\\
\cline{2-6}
& week & people & police & percent & trial \\
& made & years & people & year & court \\
& president & year & killed & prices & charges \\
\smash{\rotatebox[origin=c]{90}{\parbox{20ex}{\centering AP\\partially collapsed\\$\varpi = 100\%$}}} & officials & time & man & economic & case \\
& tuesday & don & officials & economy & judge \\
& million & back & city & rate & attorney \\
& thursday & day & shot & increase & prison \\
& national & home & authorities & report & jury \\
\\
\end{tabular}
\begin{tabular}{c C{0.145\textwidth} C{0.145\textwidth} C{0.145\textwidth} C{0.145\textwidth} C{0.145\textwidth}}
\cline{2-6}
$k$ & Topic 54 & Topic 55 & Topic 56 & Topic 57 & Topic 58
\\
$n_{k,\.}$ & 1 055 & 1 032 & 1 025 & 1 014 & 1 013
\\
\cline{2-6}
& children & north & hostages & aids & percent \\
& parents & walsh & red & virus & poll \\
& child & reagan & release & blood & survey \\
\smash{\rotatebox[origin=c]{90}{\parbox{20ex}{\centering AP\\partially collapsed\\$\varpi = 75\%$}}} & ms & iran & held & disease & points \\
& year & contra & hostage & drug & found \\
& mother & documents & anderson & infected & surveys \\
& boys & gesell & gunmen & immune & margin \\
& girl & arms & thursday & health & reported \\
\\
\end{tabular}
\begin{tabular}{c C{0.145\textwidth} C{0.145\textwidth} C{0.145\textwidth} C{0.145\textwidth} C{0.145\textwidth}}
\cline{2-6}
$k$ & Topic 108 & Topic 109 & Topic 110 & Topic 111 & Topic 112
\\
$n_{k,\.}$ & 473 & 472 & 451 & 446 & 436
\\
\cline{2-6}
& abortion & women & solidarity & waste & train \\
& souter & club & walesa & garbage & railroad \\
& anti & members & poland & recycling & cars \\
\smash{\rotatebox[origin=c]{90}{\parbox{20ex}{\centering AP\\partially collapsed\\$\varpi = 50\%$}}} & state & men & polish & city & trains \\
& women & male & government & ash & transportatio \\
& abortions & membership & mazowiecki & trash & skinner \\
& rights & female & jaruzelski & state & transit \\
& hampshire & black & talks & dump & policy \\
\\
\end{tabular}
\begin{tabular}{c C{0.145\textwidth} C{0.145\textwidth} C{0.145\textwidth} C{0.145\textwidth} C{0.145\textwidth}}
\cline{2-6}
$k$ & Topic 162 & Topic 163 & Topic 164 & Topic 165 & Topic 166
\\
$n_{k,\.}$ & 193 & 187 & 185 & 184 & 184
\\
\cline{2-6}
& health & wine & miners & dixon & barry \\
& care & warmus & mine & yates & moore \\
& spe & solomon & coal & count & jackson \\
\smash{\rotatebox[origin=c]{90}{\parbox{20ex}{\centering AP\\partially collapsed\\$\varpi = 25\%$}}} & bc & california & mines & tosh & mayor \\
& weight & bar & hull & rogers & statehood \\
& american & gallo & pittston & rig & mr \\
& diet & test & benefits & russell & gregory \\
& cholesterol & questions & platform & cookies & room \\
\\
\end{tabular}
\begin{tabular}{c C{0.145\textwidth} C{0.145\textwidth} C{0.145\textwidth} C{0.145\textwidth} C{0.145\textwidth}}
\cline{2-6}
$k$ & Topic 206 & Topic 207 & Topic 208 & Topic 209 & Topic 210
\\
$n_{k,\.}$ & 117 & 115 & 112 & 111 & 111
\\
\cline{2-6}
& pageant & mall & roberts & stuart & gold \\
& miss & malls & shell & lawn & polaroid \\
& cereal & pinochet & boigny & dea & shamrock \\
\smash{\rotatebox[origin=c]{90}{\parbox{20ex}{\centering AP\\partially collapsed\\$\varpi = 5\%$}}} & boxes & shopping & houphouet & boston & fields \\
& contestants & downtown & travelers & ruth & consolidated \\
& box & park & leonard & yankees & suit \\
& america & oak & arsenal & foundation & proposals \\
& bruce & usa & oil & richman & mining \\
\\
\end{tabular}

\end{center}

\begin{center}
\captionfont
\renewcommand{\arraystretch}{0.95}
\begin{tabular}{c C{0.145\textwidth} C{0.145\textwidth} C{0.145\textwidth} C{0.145\textwidth} C{0.145\textwidth}}
\cline{2-6}
$k$ & Topic 1 & Topic 2 & Topic 3 & Topic 4 & Topic 5
\\
$n_{k,\.}$ & 90 497 & 18 626 & 10 832 & 9 923 & 9 430
\\
\cline{2-6}
& year & years & police & dollar & percent \\
& people & year & people & market & year \\
& time & people & killed & stock & rose \\
\smash{\rotatebox[origin=c]{90}{\parbox{20ex}{\centering AP\\direct assignment\\$\varpi = 100\%$}}} & president & time & government & yen & sales \\
& years & don & reported & index & million \\
& made & home & today & late & billion \\
& state & day & capital & trading & month \\
& week & back & violence & exchange & reported \\
\\
\end{tabular}
\begin{tabular}{c C{0.145\textwidth} C{0.145\textwidth} C{0.145\textwidth} C{0.145\textwidth} C{0.145\textwidth}}
\cline{2-6}
$k$ & Topic 93 & Topic 94 & Topic 95 & Topic 96 & Topic 97
\\
$n_{k,\.}$ & 784 & 757 & 753 & 745 & 738
\\
\cline{2-6}
& keating & bus & eastern & united & smoking \\
& deconcini & driver & pilots & states & cigarettes \\
& lincoln & train & airline & nations & farmers \\
\smash{\rotatebox[origin=c]{90}{\parbox{20ex}{\centering AP\\direct assignment\\$\varpi = 75\%$}}} & senators & greyhound & orion & resolution & tobacco \\
& regulators & accident & air & international & ban \\
& meeting & passengers & union & plo & insurance \\
& committee & railroad & airlines & mission & batus \\
& gray & passenger & service & assembly & smokers \\
\\
\end{tabular}
\begin{tabular}{c C{0.145\textwidth} C{0.145\textwidth} C{0.145\textwidth} C{0.145\textwidth} C{0.145\textwidth}}
\cline{2-6}
$k$ & Topic 186 & Topic 187 & Topic 188 & Topic 189 & Topic 190
\\
$n_{k,\.}$ & 346 & 338 & 338 & 338 & 334
\\
\cline{2-6}
& power & cable & conservatives & water & dental \\
& franc & television & flag & dam & funds \\
& jersey & nbc & conservative & river & claims \\
\smash{\rotatebox[origin=c]{90}{\parbox{20ex}{\centering AP\\direct assignment\\$\varpi = 50\%$}}} & bradley & tempo & amendment & area & plough \\
& utility & hsn & speaker & reservoir & oral \\
& wppss & industry & darman & savannah & counter \\
& utilities & subscribers & kemp & corps & embassy \\
& west & tv & republicans & canyon & mid \\
\\
\end{tabular}
\begin{tabular}{c C{0.145\textwidth} C{0.145\textwidth} C{0.145\textwidth} C{0.145\textwidth} C{0.145\textwidth}}
\cline{2-6}
$k$ & Topic 279 & Topic 280 & Topic 281 & Topic 282 & Topic 283
\\
$n_{k,\.}$ & 220 & 219 & 219 & 219 & 218
\\
\cline{2-6}
& fernandez & water & bloom & canadian & election \\
& fdic & lake & minnick & lee & grenada \\
& republicbank & mussels & walters & ritalin & boigny \\
\smash{\rotatebox[origin=c]{90}{\parbox{20ex}{\centering AP\\direct assignment\\$\varpi = 25\%$}}} & weicker & neill & lawyer & murphy & houphouet \\
& virginia & erie & athletes & domestic & gairy \\
& ruth & problem & college & security & coast \\
& robinson & plant & suspect & woods & nov \\
& station & north & signing & radio & failed \\
\\
\end{tabular}
\begin{tabular}{c C{0.145\textwidth} C{0.145\textwidth} C{0.145\textwidth} C{0.145\textwidth} C{0.145\textwidth}}
\cline{2-6}
$k$ & Topic 354 & Topic 355 & Topic 356 & Topic 357 & Topic 358
\\
$n_{k,\.}$ & 133 & 133 & 133 & 132 & 132
\\
\cline{2-6}
& machine & young & count & reynolds & turkey \\
& stop & johnston & forman & premier & department \\
& reed & golf & festival & bond & bird \\
\smash{\rotatebox[origin=c]{90}{\parbox{20ex}{\centering AP\\direct assignment\\$\varpi = 5\%$}}} & gun & notes & rig & release & cooking \\
& chief & bodies & arts & news & wash \\
& sununu & homes & hughes & regulated & bacteria \\
& geneva & call & lights & address & stuffed \\
& formal & shortage & staged & petition & adams \\
\\
\end{tabular}

\end{center}

\newpage
\section{Appendix: quantile summary of topics for \textsc{CGCBIB}}
\label{apdx:cgcbib}

Here we display a multi-quantile summary for CGCBIB, obtained by ranking all topics with at least 100 tokens by their total number of tokens, computing the $\varpi=100\%$, $75\%$, $50\%$, $25\%$, and $5\%$ quantiles.
We compute the five topics closest to each quantile by number of tokens, and display their top-eight words.

\begin{center}
\captionfont
\renewcommand{\arraystretch}{0.95}
\begin{tabular}{c C{0.145\textwidth} C{0.145\textwidth} C{0.145\textwidth} C{0.145\textwidth} C{0.145\textwidth}}
\cline{2-6}
$k$ & Topic 1 & Topic 2 & Topic 3 & Topic 4 & Topic 5
\\
$n_{k,\.}$ & 110 702 & 58 811 & 27 084 & 21 215 & 19 832
\\
\cline{2-6}
& elegans & elegans & elegans & gene & mutations \\
& caenorhabditi & protein & genetic & elegans & gene \\
& nematode & caenorhabditi & molecular & sequence & mutants \\
\smash{\rotatebox[origin=c]{90}{\parbox{20ex}{\centering CGCBIB\\partially collapsed\\$\varpi = 100\%$}}} & results & gene & development & protein & genes \\
& found & function & caenorhabditi & caenorhabditi & mutant \\
& show & proteins & nematode & amino & elegans \\
& observed & required & studies & cdna & caenorhabditi \\
& specific & show & model & acid & alleles \\
\\
\end{tabular}
\begin{tabular}{c C{0.145\textwidth} C{0.145\textwidth} C{0.145\textwidth} C{0.145\textwidth} C{0.145\textwidth}}
\cline{2-6}
$k$ & Topic 54 & Topic 55 & Topic 56 & Topic 57 & Topic 58
\\
$n_{k,\.}$ & 2 166 & 2 061 & 2 048 & 2 040 & 2 025
\\
\cline{2-6}
& germ & egl & emb & spe & wnt \\
& germline & egg & temperature & sperm & mom \\
& early & laying & mutants & fer & pop \\
\smash{\rotatebox[origin=c]{90}{\parbox{20ex}{\centering CGCBIB\\partially collapsed\\$\varpi = 75\%$}}} & granules & serotonin & sensitive & spermatozoa & signaling \\
& cells & neurons & zyg & membrane & bar \\
& embryos & cat & maternal & spermatids & pathway \\
& somatic & dopamine & expression & spermatogenes & lin \\
& line & mutants & embryonic & pseudopod & wrm \\
\\
\end{tabular}
\begin{tabular}{c C{0.145\textwidth} C{0.145\textwidth} C{0.145\textwidth} C{0.145\textwidth} C{0.145\textwidth}}
\cline{2-6}
$k$ & Topic 109 & Topic 110 & Topic 111 & Topic 112 & Topic 113
\\
$n_{k,\.}$ & 930 & 916 & 915 & 900 & 893
\\
\cline{2-6}
& vit & binding & kinesin & growth & eat \\
& yolk & affinity & klp & survival & pharyngeal \\
& vitellogenin & site & transport & mortality & pharynx \\
\smash{\rotatebox[origin=c]{90}{\parbox{20ex}{\centering CGCBIB\\partially collapsed\\$\varpi = 50\%$}}} & genes & activity & motor & population & pumping \\
& yp & sites & ift & rate & inx \\
& proteins & avermectin & cilia & populations & gap \\
& vpe & elegans & dynein & parameter & feeding \\
& lrp & membrane & movement & size & junctions \\
\\
\end{tabular}
\begin{tabular}{c C{0.145\textwidth} C{0.145\textwidth} C{0.145\textwidth} C{0.145\textwidth} C{0.145\textwidth}}
\cline{2-6}
$k$ & Topic 164 & Topic 165 & Topic 166 & Topic 167 & Topic 168
\\
$n_{k,\.}$ & 386 & 369 & 368 & 364 & 360
\\
\cline{2-6}
& mlc & dom & innate & vha & ife \\
& mel & effects & immune & atpase & cap \\
& myosin & humic & immunity & subunit & eife \\
\smash{\rotatebox[origin=c]{90}{\parbox{20ex}{\centering CGCBIB\\partially collapsed\\$\varpi = 25\%$}}} & nmy & pyrene & abf & genes & capping \\
& chain & effect & lys & vacuolar & cel \\
& elongation & bioconcentrat & toll & subunits & gtp \\
& rho & dissolved & antimicrobial & atpases & isoforms \\
& phosphatase & substances & pathway & type & rna \\
\\
\end{tabular}
\begin{tabular}{c C{0.145\textwidth} C{0.145\textwidth} C{0.145\textwidth} C{0.145\textwidth} C{0.145\textwidth}}
\cline{2-6}
$k$ & Topic 208 & Topic 209 & Topic 210 & Topic 211 & Topic 212
\\
$n_{k,\.}$ & 141 & 141 & 140 & 140 & 136
\\
\cline{2-6}
& ubq & asp & da & ion & hcf \\
& gc & salmonella & cl & diet & cehcf \\
& tbp & poona & fli & relative & vp \\
\smash{\rotatebox[origin=c]{90}{\parbox{20ex}{\centering CGCBIB\\partially collapsed\\$\varpi = 5\%$}}} & footprints & enterica & gs & xpa & ldb \\
& oscillin & clp & db & groups & cell \\
& tlf & serotype & glu & carbon & mammalian \\
& ubiquitin & necrotic & phospholipid & characteristi & phosphorylati \\
& tata & mug & tg & atoms & neural \\
\\
\end{tabular}

\end{center}

\begin{center}
\captionfont
\renewcommand{\arraystretch}{0.95}
\begin{tabular}{c C{0.145\textwidth} C{0.145\textwidth} C{0.145\textwidth} C{0.145\textwidth} C{0.145\textwidth}}
\cline{2-6}
$k$ & Topic 1 & Topic 2 & Topic 3 & Topic 4 & Topic 5
\\
$n_{k,\.}$ & 65 059 & 41 005 & 33 714 & 27 221 & 22 813
\\
\cline{2-6}
& elegans & elegans & elegans & mutations & elegans \\
& caenorhabditi & genetic & caenorhabditi & elegans & gene \\
& protein & caenorhabditi & nematode & gene & sequence \\
\smash{\rotatebox[origin=c]{90}{\parbox{20ex}{\centering CGCBIB\\direct assignment\\$\varpi = 100\%$}}} & gene & nematode & results & mutants & caenorhabditi \\
& function & molecular & observed & genes & protein \\
& proteins & development & high & caenorhabditi & amino \\
& required & studies & type & mutant & cdna \\
& show & model & effect & function & acid \\
\\
\end{tabular}
\begin{tabular}{c C{0.145\textwidth} C{0.145\textwidth} C{0.145\textwidth} C{0.145\textwidth} C{0.145\textwidth}}
\cline{2-6}
$k$ & Topic 68 & Topic 69 & Topic 70 & Topic 71 & Topic 72
\\
$n_{k,\.}$ & 1 921 & 1 894 & 1 836 & 1 828 & 1 776
\\
\cline{2-6}
& loci & worm & cell & alpha & unc \\
& genetic & elegans & epithelial & gpa & gaba \\
& strains & research & junctions & egl & receptors \\
\smash{\rotatebox[origin=c]{90}{\parbox{20ex}{\centering CGCBIB\\direct assignment\\$\varpi = 75\%$}}} & lines & caenorhabditi & membrane & signaling & receptor \\
& life & brenner & cells & protein & resistance \\
& mutations & years & dlg & goa & lev \\
& mutation & nematode & hmp & rgs & levamisole \\
& inbred & biology & exc & proteins & cholinergic \\
\\
\end{tabular}
\begin{tabular}{c C{0.145\textwidth} C{0.145\textwidth} C{0.145\textwidth} C{0.145\textwidth} C{0.145\textwidth}}
\cline{2-6}
$k$ & Topic 137 & Topic 138 & Topic 139 & Topic 140 & Topic 141
\\
$n_{k,\.}$ & 782 & 779 & 779 & 763 & 756
\\
\cline{2-6}
& cell & hsp & survival & acid & yeast \\
& dimensional & heat & mortality & amino & cerevisiae \\
& microscopy & shock & model & acids & saccharomyces \\
\smash{\rotatebox[origin=c]{90}{\parbox{20ex}{\centering CGCBIB\\direct assignment\\$\varpi = 50\%$}}} & embryo & chaperone & data & nematode & pombe \\
& analysis & small & gompertz & glycine & cell \\
& system & proteins & parameter & briggsae & budding \\
& computer & crystallin & population & cytochrome & schizosacchar \\
& time & hsps & rate & multiple & cycle \\
\\
\end{tabular}
\begin{tabular}{c C{0.145\textwidth} C{0.145\textwidth} C{0.145\textwidth} C{0.145\textwidth} C{0.145\textwidth}}
\cline{2-6}
$k$ & Topic 206 & Topic 207 & Topic 208 & Topic 209 & Topic 210
\\
$n_{k,\.}$ & 364 & 359 & 358 & 343 & 343
\\
\cline{2-6}
& activity & pgp & telomere & gcy & mediator \\
& lh & mrp & telomeres & guanylyl & med \\
& activities & aat & ceh & cyclase & sop \\
\smash{\rotatebox[origin=c]{90}{\parbox{20ex}{\centering CGCBIB\\direct assignment\\$\varpi = 25\%$}}} & juvenile & cells & yeast & wee & transcription \\
& nematodes & glycoprotein & nematode & ase & development \\
& antiallatal & mammalian & mrt & receptor & pvl \\
& hormone & resistance & telomerase & cyclases & transcription \\
& insect & glycoproteins & telomeric & gfp & dhp \\
\\
\end{tabular}
\begin{tabular}{c C{0.145\textwidth} C{0.145\textwidth} C{0.145\textwidth} C{0.145\textwidth} C{0.145\textwidth}}
\cline{2-6}
$k$ & Topic 261 & Topic 262 & Topic 263 & Topic 264 & Topic 265
\\
$n_{k,\.}$ & 164 & 164 & 159 & 157 & 156
\\
\cline{2-6}
& cog & atp & calcineurin & selection & srl \\
& wd & structures & egg & flow & rol \\
& repeat & oligomerizati & bovine & separation & threshold \\
\smash{\rotatebox[origin=c]{90}{\parbox{20ex}{\centering CGCBIB\\direct assignment\\$\varpi = 5\%$}}} & connection & family & laying & redundancy & ra \\
& native & binding & hg & flows & energy \\
& response & members & white & directional & free \\
& worm & stability & haemin & solution & external \\
& nr & mechanism & phosphatase & period & experimental \\
\\
\end{tabular}

\end{center}

\newpage
\section{Appendix: quantile summary of topics for \textsc{NeurIPS}}
\label{apdx:neurips}

Here we display a multi-quantile summary for NeurIPS, obtained by ranking all topics with at least 100 tokens by their total number of tokens, computing the $\varpi=100\%$, $75\%$, $50\%$, $25\%$, and $5\%$ quantiles.
We compute the five topics closest to each quantile by number of tokens, and display their top-eight words.

\begin{center}
\captionfont
\renewcommand{\arraystretch}{0.95}
\begin{tabular}{c C{0.145\textwidth} C{0.145\textwidth} C{0.145\textwidth} C{0.145\textwidth} C{0.145\textwidth}}
\cline{2-6}
$k$ & Topic 1 & Topic 2 & Topic 3 & Topic 4 & Topic 5
\\
$n_{k,\.}$ & 182 743 & 162 355 & 129 745 & 52 356 & 44 155
\\
\cline{2-6}
& system & function & number & model & training \\
& information & case & result & neural & set \\
& approach & result & small & result & data \\
\smash{\rotatebox[origin=c]{90}{\parbox{20ex}{\centering NeurIPS\\partially collapsed\\$\varpi = 100\%$}}} & set & term & values & system & test \\
& problem & parameter & order & activity & performance \\
& research & neural & large & input & number \\
& computer & form & effect & pattern & result \\
& single & defined & high & function & error \\
\\
\end{tabular}
\begin{tabular}{c C{0.145\textwidth} C{0.145\textwidth} C{0.145\textwidth} C{0.145\textwidth} C{0.145\textwidth}}
\cline{2-6}
$k$ & Topic 148 & Topic 149 & Topic 150 & Topic 151 & Topic 152
\\
$n_{k,\.}$ & 2 585 & 2 585 & 2 574 & 2 559 & 2 549
\\
\cline{2-6}
& genetic & delay & bengio & fig & matching \\
& algorithm & bifurcation & output & properties & model \\
& population & oscillation & dependencies & proc & point \\
\smash{\rotatebox[origin=c]{90}{\parbox{20ex}{\centering NeurIPS\\partially collapsed\\$\varpi = 75\%$}}} & fitness & point & input & step & correspondenc \\
& string & stability & experiment & range & match \\
& generation & fixed & frasconi & structure & problem \\
& bit & limit & term & calculation & set \\
& function & hopf & information & illinois & object \\
\\
\end{tabular}
\begin{tabular}{c C{0.145\textwidth} C{0.145\textwidth} C{0.145\textwidth} C{0.145\textwidth} C{0.145\textwidth}}
\cline{2-6}
$k$ & Topic 297 & Topic 298 & Topic 299 & Topic 300 & Topic 301
\\
$n_{k,\.}$ & 1 310 & 1 309 & 1 309 & 1 297 & 1 295
\\
\cline{2-6}
& vor & routing & speaker & delay & memory \\
& storage & load & recognition & input & action \\
& anastasio & network & normalization & transition & states \\
\smash{\rotatebox[origin=c]{90}{\parbox{20ex}{\centering NeurIPS\\partially collapsed\\$\varpi = 50\%$}}} & responses & path & male & window & agent \\
& velocity & packet & feature & width & sensing \\
& pan & traffic & female & connection & loop \\
& rotation & shortest & mntn & information & history \\
& vestibular & policy & ntn & temporal & mdp \\
\\
\end{tabular}
\begin{tabular}{c C{0.145\textwidth} C{0.145\textwidth} C{0.145\textwidth} C{0.145\textwidth} C{0.145\textwidth}}
\cline{2-6}
$k$ & Topic 446 & Topic 447 & Topic 448 & Topic 449 & Topic 450
\\
$n_{k,\.}$ & 748 & 748 & 746 & 739 & 735
\\
\cline{2-6}
& composite & psom & limited & tau & cmm \\
& mdp & robot & interconnect & hypothesis & speed \\
& action & camera & fan & mansour & particle \\
\smash{\rotatebox[origin=c]{90}{\parbox{20ex}{\centering NeurIPS\\partially collapsed\\$\varpi = 25\%$}}} & elemental & set & shunting & growth & particles \\
& optimal & pointing & modularity & coefficient & pattern \\
& payoff & coordinates & collective & function & presence \\
& solution & basis & linear & stem & method \\
& mdt & ritter & unit & large & card \\
\\
\end{tabular}
\begin{tabular}{c C{0.145\textwidth} C{0.145\textwidth} C{0.145\textwidth} C{0.145\textwidth} C{0.145\textwidth}}
\cline{2-6}
$k$ & Topic 566 & Topic 567 & Topic 568 & Topic 569 & Topic 570
\\
$n_{k,\.}$ & 396 & 385 & 383 & 379 & 372
\\
\cline{2-6}
& morph & minimal & visualization & periodic & machine \\
& kernel & root & high & period & capacity \\
& parent & biases & low & coefficient & path \\
\smash{\rotatebox[origin=c]{90}{\parbox{20ex}{\centering NeurIPS\\partially collapsed\\$\varpi = 5\%$}}} & human & attribute & diagram & primitive & trouble \\
& busey & remove & visualizing & homogeneous & high \\
& similar & rumelhart & graphic & tst & task \\
& exemplar & row & fund & mhaskar & increasing \\
& distinctivene & exponential & window & chain & measures \\
\\
\end{tabular}

\end{center}

\begin{center}
\captionfont
\renewcommand{\arraystretch}{0.95}
\begin{tabular}{c C{0.145\textwidth} C{0.145\textwidth} C{0.145\textwidth} C{0.145\textwidth} C{0.145\textwidth}}
\cline{2-6}
$k$ & Topic 6 & Topic 2 & Topic 1 & Topic 13 & Topic 62
\\
$n_{k,\.}$ & 473 770 & 93 435 & 52 418 & 50 965 & 41 565
\\
\cline{2-6}
& network & network & model & model & function \\
& model & unit & neuron & data & network \\
& learning & input & input & parameter & bound \\
\smash{\rotatebox[origin=c]{90}{\parbox{25ex}{\centering NeurIPS\\subcluster split-merge\,\,\,\,\,\,\,\\$\varpi = 100\%$}}} & function & learning & network & network & dimension \\
& input & training & cell & algorithm & learning \\
& neural & weight & system & mixture & result \\
& algorithm & neural & unit & function & number \\
& set & output & visual & gaussian & set \\
\\
\end{tabular}
\begin{tabular}{c C{0.145\textwidth} C{0.145\textwidth} C{0.145\textwidth} C{0.145\textwidth} C{0.145\textwidth}}
\cline{2-6}
$k$ & Topic 440 & Topic 170 & Topic 334 & Topic 418 & Topic 312
\\
$n_{k,\.}$ & 2 678 & 2 657 & 2 643 & 2 636 & 2 622
\\
\cline{2-6}
& learning & movement & motion & learning & cell \\
& critic & visual & unit & algorithm & correlation \\
& function & vector & direction & action & neuron \\
\smash{\rotatebox[origin=c]{90}{\parbox{25ex}{\centering NeurIPS\\subcluster split-merge\,\,\,\,\,\,\,\\$\varpi = 75\%$}}} & actor & image & model & advantage & model \\
& algorithm & model & stage & system & unit \\
& system & location & input & function & interaction \\
& control & eye & network & policy & firing \\
& model & map & cell & control & set \\
\\
\end{tabular}
\begin{tabular}{c C{0.145\textwidth} C{0.145\textwidth} C{0.145\textwidth} C{0.145\textwidth} C{0.145\textwidth}}
\cline{2-6}
$k$ & Topic 378 & Topic 322 & Topic 82 & Topic 344 & Topic 414
\\
$n_{k,\.}$ & 1 032 & 1 028 & 1 013 & 1 009 & 1 006
\\
\cline{2-6}
& iiii & cell & model & form & component \\
& cell & spike & response & word & algorithm \\
& network & unit & neural & phone & sources \\
\smash{\rotatebox[origin=c]{90}{\parbox{25ex}{\centering NeurIPS\\subcluster split-merge\,\,\,\,\,\,\,\\$\varpi = 50\%$}}} & neural & function & escape & input & analysis \\
& response & firing & interneuron & network & data \\
& model & result & cockroach & system & noise \\
& point & transfer & leg & training & orientation \\
& fixed & sorting & input & meaning & spatial \\
\\
\end{tabular}
\begin{tabular}{c C{0.145\textwidth} C{0.145\textwidth} C{0.145\textwidth} C{0.145\textwidth} C{0.145\textwidth}}
\cline{2-6}
$k$ & Topic 220 & Topic 341 & Topic 441 & Topic 447 & Topic 308
\\
$n_{k,\.}$ & 728 & 723 & 723 & 722 & 721
\\
\cline{2-6}
& aspect & element & network & input & traffic \\
& object & pairing & neural & unit & waiting \\
& view & grouping & constraint & spike & elevator \\
\smash{\rotatebox[origin=c]{90}{\parbox{25ex}{\centering NeurIPS\\subcluster split-merge\,\,\,\,\,\,\,\\$\varpi = 25\%$}}} & node & group & match & layer & appeared \\
& learning & saliency & learn & learning & application \\
& network & contour & problem & model & compared \\
& weight & computation & initial & predict & department \\
& equation & optimal & row & prediction & found \\
\\
\end{tabular}
\begin{tabular}{c C{0.145\textwidth} C{0.145\textwidth} C{0.145\textwidth} C{0.145\textwidth} C{0.145\textwidth}}
\cline{2-6}
$k$ & Topic 259 & Topic 246 & Topic 195 & Topic 245 & Topic 293
\\
$n_{k,\.}$ & 509 & 507 & 506 & 503 & 503
\\
\cline{2-6}
& input & network & network & network & network \\
& output & neural & symbol & equation & function \\
& activation & task & vtp & neuron & adaptation \\
\smash{\rotatebox[origin=c]{90}{\parbox{25ex}{\centering NeurIPS\\subcluster split-merge\,\,\,\,\,\,\,\\$\varpi = 5\%$}}} & data & link & learning & moment & algorithm \\
& encoded & food & phrases & neural & prediction \\
& function & nodes & sentences & approximation & projection \\
& hidden & output & vpp & ohira & neural \\
& model & recurrent & classificatio & stochastic & training \\
\\
\end{tabular}

\end{center}

\newpage
\section{Appendix: topics produced by Algorithm \ref{alg:ppu-hdp} on \textsc{PubMed}}
\label{apdx:pubmed}

Here we show top eight words for each topic together with total number of tokens assigned, which is shown at the top of each table.
We display all topics containing at least eight unique word tokens.

\begin{center}
\captionfont
\renewcommand{\arraystretch}{0.95}
\begin{tabular}{c C{0.145\textwidth} C{0.145\textwidth} C{0.145\textwidth} C{0.145\textwidth} C{0.145\textwidth}}
\cline{2-6}
$k$ & Topic 1 & Topic 2 & Topic 3 & Topic 4 & Topic 5
\\
$n_{k,\.}$ & 47 322 709 & 40 229 486 & 34 685 122 & 30 795 166 & 30 707 144
\\
\cline{2-6}
& care & age & model & cell & gene \\
& health & risk & data & expression & protein \\
& patient & children & system & growth & dna \\
\smash{\rotatebox[origin=c]{90}{\parbox{20ex}{\centering PubMed}}} & medical & year & time & protein & expression \\
& research & women & analysis & factor & sequence \\
& clinical & patient & effect & receptor & genes \\
& system & factor & test & kinase & rna \\
& cost & population & field & beta & region \\
\\
\end{tabular}
\begin{tabular}{c C{0.145\textwidth} C{0.145\textwidth} C{0.145\textwidth} C{0.145\textwidth} C{0.145\textwidth}}
\cline{2-6}
$k$ & Topic 6 & Topic 7 & Topic 8 & Topic 9 & Topic 10
\\
$n_{k,\.}$ & 28 510 997 & 27 277 306 & 26 709 116 & 26 408 263 & 25 200 662
\\
\cline{2-6}
& cell & cancer & patient & rat & cell \\
& il & tumor & treatment & receptor & electron \\
& cd & patient & mg & effect & muscle \\
\smash{\rotatebox[origin=c]{90}{\parbox{20ex}{\centering PubMed}}} & mice & carcinoma & drug & neuron & tissue \\
& antigen & cell & effect & brain & fiber \\
& human & breast & therapy & activity & rat \\
& lymphocytes & survival & dose & stimulation & development \\
& immune & tumour & day & induced & microscopy \\
\\
\end{tabular}
\begin{tabular}{c C{0.145\textwidth} C{0.145\textwidth} C{0.145\textwidth} C{0.145\textwidth} C{0.145\textwidth}}
\cline{2-6}
$k$ & Topic 11 & Topic 12 & Topic 13 & Topic 14 & Topic 15
\\
$n_{k,\.}$ & 24 856 624 & 24 750 437 & 24 607 618 & 24 482 090 & 22 956 810
\\
\cline{2-6}
& patient & patient & blood & patient & infection \\
& surgery & artery & pressure & disease & virus \\
& complication & heart & flow & clinical & hiv \\
\smash{\rotatebox[origin=c]{90}{\parbox{20ex}{\centering PubMed}}} & surgical & coronary & min & diagnosis & strain \\
& treatment & ventricular & effect & lesion & infected \\
& year & myocardial & exercise & brain & patient \\
& postoperative & cardiac & arterial & syndrome & positive \\
& operation & left & heart & imaging & viral \\
\\
\end{tabular}
\begin{tabular}{c C{0.145\textwidth} C{0.145\textwidth} C{0.145\textwidth} C{0.145\textwidth} C{0.145\textwidth}}
\cline{2-6}
$k$ & Topic 16 & Topic 17 & Topic 18 & Topic 19 & Topic 20
\\
$n_{k,\.}$ & 22 095 623 & 21 838 239 & 21 363 408 & 20 887 061 & 20 828 980
\\
\cline{2-6}
& ca & structure & concentration & pregnancy & protein \\
& effect & binding & degrees & level & binding \\
& receptor & protein & samples & women & human \\
\smash{\rotatebox[origin=c]{90}{\parbox{20ex}{\centering PubMed}}} & channel & reaction & liquid & hormone & antibodies \\
& cell & acid & solution & day & acid \\
& calcium & interaction & assay & fetal & alpha \\
& concentration & compound & detection & infant & antibody \\
& na & site & system & concentration & gel \\
\\
\end{tabular}
\begin{tabular}{c C{0.145\textwidth} C{0.145\textwidth} C{0.145\textwidth} C{0.145\textwidth} C{0.145\textwidth}}
\cline{2-6}
$k$ & Topic 21 & Topic 22 & Topic 23 & Topic 24 & Topic 25
\\
$n_{k,\.}$ & 20 106 260 & 19 788 488 & 18 675 096 & 17 163 327 & 16 440 018
\\
\cline{2-6}
& rat & bone & patient & gene & activity \\
& cell & patient & renal & mutation & acid \\
& effect & joint & liver & genetic & enzyme \\
\smash{\rotatebox[origin=c]{90}{\parbox{20ex}{\centering PubMed}}} & liver & muscle & transplantati & chromosome & liver \\
& mice & fractures & blood & analysis & concentration \\
& dose & hip & disease & genes & rat \\
& drug & year & acute & dna & enzymes \\
& mg & implant & chronic & polymorphism & synthesis \\
\\
\end{tabular}
\begin{tabular}{c C{0.145\textwidth} C{0.145\textwidth} C{0.145\textwidth} C{0.145\textwidth} C{0.145\textwidth}}
\cline{2-6}
$k$ & Topic 26 & Topic 27 & Topic 28 & Topic 29 & Topic 30
\\
$n_{k,\.}$ & 16 136 164 & 14 201 063 & 13 706 016 & 13 191 158 & 13 105 245
\\
\cline{2-6}
& effect & diet & patient & strain & protein \\
& platelet & weight & disease & plant & membrane \\
& induced & intake & gastric & growth & cell \\
\smash{\rotatebox[origin=c]{90}{\parbox{20ex}{\centering PubMed}}} & oxide & food & asthma & acid & domain \\
& rat & body & test & bacteria & binding \\
& cell & effect & pylori & activity & receptor \\
& endothelial & acid & arthritis & cell & lipid \\
& activity & vitamin & chronic & species & membranes \\
\\
\end{tabular}
\begin{tabular}{c C{0.145\textwidth} C{0.145\textwidth} C{0.145\textwidth} C{0.145\textwidth} C{0.145\textwidth}}
\cline{2-6}
$k$ & Topic 31 & Topic 32 & Topic 33 & Topic 34 & Topic 35
\\
$n_{k,\.}$ & 12 705 261 & 12 624 252 & 10 422 885 & 9 850 167 & 7 027 660
\\
\cline{2-6}
& insulin & species & exposure & skin & level \\
& glucose & population & concentration & patient & patient \\
& diabetes & infection & iron & eyes & ml \\
\smash{\rotatebox[origin=c]{90}{\parbox{20ex}{\centering PubMed}}} & cholesterol & animal & level & eye & control \\
& level & egg & water & retinal & serum \\
& diabetic & host & effect & laser & plasma \\
& plasma & parasite & exposed & visual & factor \\
& lipoprotein & malaria & lead & corneal & concentration \\
\\
\end{tabular}
\begin{tabular}{c C{0.145\textwidth} C{0.145\textwidth} C{0.145\textwidth} C{0.145\textwidth} C{0.145\textwidth}}
\cline{2-6}
$k$ & Topic 36 & Topic 37 & Topic 38 & Topic 39 & Topic 40
\\
$n_{k,\.}$ & 6 130 945 & 644 182 & 2 264 & 1 325 & 104
\\
\cline{2-6}
& dental & sleep & ppr & pac & feather \\
& oral & caffeine & csc & foal & tieg \\
& teeth & tea & stretch & cpr & sorghum \\
\smash{\rotatebox[origin=c]{90}{\parbox{20ex}{\centering PubMed}}} & tooth & effect & pthrp & pacap & coii \\
& periodontal & theophylline & response & edm & phycocyanin \\
& treatment & night & br & speck & vanx \\
& salivary & coffee & gei & branchial & midrib \\
& gland & green & pth & lth & ifi \\
\\
\end{tabular}
\begin{tabular}{c C{0.145\textwidth} C{0.145\textwidth} C{0.145\textwidth} C{0.145\textwidth} C{0.145\textwidth}}
\cline{2-2}
$k$ & Topic 41 &  &  &  & 
\\
$n_{k,\.}$ & 104 &  &  &  & 
\\
\cline{2-2}
& steer \\
& mca \\
& persistency \\
\smash{\rotatebox[origin=c]{90}{\parbox{20ex}{\centering PubMed}}} & buckwheat \\
& dnak \\
& eset \\
& branding \\
& akr \\
\\
\end{tabular}

\end{center}

\end{document}